\documentclass{article}
\usepackage{spconf,amsmath,graphicx}
\usepackage{amssymb,amsfonts,mathtools}
\usepackage{amsthm}

\usepackage{algorithmic}
\usepackage{diagbox}
\usepackage{textcomp}
\usepackage{comment}
\usepackage{xcolor}
\usepackage{paralist}
\usepackage{multirow}
\usepackage{float}
\usepackage[caption=false, font=footnotesize]{subfig}
\usepackage{color,colortbl}
\definecolor{Gray}{gray}{0.9}
\usepackage{array}
\makeatletter
\usepackage[inline]{enumitem}
\newcommand{\thickhline}{%
    \noalign {\ifnum 0=`}\fi \hrule height 1pt
    \futurelet \reserved@a \@xhline
}

\newcolumntype{"}{@{\hskip\tabcolsep\vrule width 1pt\hskip\tabcolsep}}
\makeatother
   
\usepackage{etoolbox}
\makeatletter
\patchcmd{\@makecaption}
  {\scshape}
  {}
  {}
  {}
\makeatother

\usepackage{symbol_maosheng}
\usepackage{xcolor}

\newtheorem{prop}{Proposition}

\title{Simplicial convolutional neural networks}
\name{Maosheng Yang, Elvin Isufi and Geert Leus\thanks{Faculty of Electrical Engineering, Mathematics and Computer Science, Delft University of Technology, Delft, The Netherlands. This work is supported by the TU Delft AI Labs Programme. e-mail: {\{m.yang-2, e.isufi, g.j.t.leus\}@tudelft.nl}.}}
\address{}
\begin{document}
%
\maketitle
\begin{abstract}
Graphs can model networked data by representing them as nodes and their pairwise relationships as edges. Recently, signal processing and neural networks have been extended to process and learn from data on graphs, with achievements in tasks like graph signal reconstruction, graph or node classifications, and link prediction. However, these methods are only suitable for data defined on the nodes of a graph. In this paper, we propose a simplicial convolutional neural network (SCNN) architecture to learn from data defined on simplices, e.g., nodes, edges, triangles, etc. We study the SCNN permutation and orientation equivariance, complexity, and spectral analysis. Finally, we test the SCNN performance for imputing citations on a coauthorship complex. 
\end{abstract}

\section{Introduction}
Graphs are powerful models to represent irregular data by encoding their pairwise relationships. To process such networked data, signal processing concepts have been extended to the graph domain, defining, for instance, the graph Fourier transform and graph filters \cite{ortega2018graph}. Meanwhile, graph neural networks (GNNs) have achieved a good performance in tasks like rating prediction in recommender systems, graph or node classification and link prediction \cite{gama2020, isufi2021accuracy, kipf2016semi,isufi2021edgenets,wu2020comprehensive}. 

However, graph signal processing and GNNs are designed for data defined on nodes of a graph. In real-world problems, we might also have data defined on edges, triangles, etc, of a network, such as communication or traffic flow in a data or road network, paper citations of a coauthorship network and so on \cite{barbarossa2020, ebli2020simplicial,leus2021topological}. To model data defined on such higher-order network structures, we can use a simplicial complex, which is a collection of simplices, i.e., nodes, edges, triangles, etc. Recently, signal processing and neural networks on simplicial complexes have emerged. In \cite{barbarossa2020}, simplicial data as well as the simplicial Fourier transform have been defined. The authors of \cite{schaub2018flow, schaub2021} provided an overview of some simplicial signal processing techniques to process flow-typed data. Our previous work \cite{yang2021finite} analyzed the definition of frequencies for simplicial signals and proposed two types of simplicial filters based on the Hodge Laplacian. 

Meanwhile, researchers have also attempted to develop neural networks on simplicial complexes. In \cite{ebli2020simplicial}, a basic simplicial neural network (SNN) was proposed with a convolutional layer composed of a basic simplicial filter \cite{yang2021finite} and a nonlinearity. Message passing neural networks (MPNNs) have been generalized to simplicial complexes in \cite{bodnar2021weisfeiler} where the aggregation and updating functions consider in addition to the edge data also data defined on adjacent simplices, i.e., nodes and triangles. The neural network architectures in \cite{bunch2020simplicial,roddenberry2021principled} are instances of \cite{bodnar2021weisfeiler} by specifying the aggregation functions as simplicial filters. Another attempt in \cite{roddenberry2019hodgenet} considered recurrent architectures in MPNNs for flow interpolation and graph classification tasks. 

Motivated by the principle of the convolution operator, in this paper we propose simplicial convolutional neural networks (SCNNs). Differently from the earlier approach in \cite{ebli2020simplicial}, we build an SCNN with simplicial filters of higher flexibility in exploiting the lower- and upper-neighbors of a simplex. And differently from the MPNN, the proposed SCNN considers multihop information exchange within a layer and enjoys spectral interpretability via the simplicial Fourier transform. Our specific contributions are:
\begin{enumerate*}[label=\roman*)]
    \item we propose an inductive SCNN based on the more advanced simplicial filter of \cite{yang2021finite} and discuss its connections to the related work; 
    \item we analyze its permutation and orientation equivariances as well as characterize the SCNN in the spectral domain;
    \item we test the performance on citation data imputation in a coauthorship complex outperforming the state-of-the-art. 
\end{enumerate*}


\section{Simplicial signal processing} \label{sec:background}
In this section, we recall some important simplicial signal processing concepts, including simplicial complexes and signals, the Hodge decomposition, and simplicial filters.  

\vspace{1mm}\noindent\textbf{Simplicial complexes and signals.}
Given a finite set of vertices $\ccalV$, a $k$-simplex $\ccalS^k$ is a subset of $\ccalV$ with cardinality $k+1$. 
A \emph{face} of $\ccalS^k$ is a subset with cardinality $k$ and thus a $k$-simplex has $k+1$ faces. A \emph{coface} of $\ccalS^k$ is a $(k+1)$-simplex that includes $\ccalS^k$ \cite{barbarossa2020, lim2015hodge}. A simplicial complex of order $K$, $\ccalX^K$, is a collection of $k$-simplices $\ccalS^k$, $k=0,\dots,K$, with an inclusion property--for any $\ccalS^k\in\ccalX^K$, then $\ccalS^{k-1}\in\ccalX^K$ if $\ccalS^{k-1}\subset\ccalS^k$. We denote the number of $k$-simplices in $\ccalX^K$ by $N_k$. If two simplices share a common face, then they are lower neighbours; if they share a common coface, they are upper neighbours \cite{yang2021finite}.
A graph is a simplicial complex where nodes are 0-simplices, and edges are 1-simplices. 

In a simplicial complex, we define a $k$-simplicial signal $\bbx^k = [x^k_1,\dots,x^k_{N_k}]$ as a mapping from the $k$-simplices to the real space $\setR^{N_k}$.
For example, $\setR^{N_0}$ is the graph signal space in GSP, and $\setR^{N_1}$ is the space of edge flows. 
For an edge flow $\bbx^1\in\setR^{N_1}$, the sign of its entry denotes the direction of the flow relative to a chosen reference orientation \cite{schaub2021,lim2015hodge}. 

\vspace{1mm}\noindent\textbf{Hodge Laplacian and decomposition.} 
We represent the relations between $(k-1)$- and $k$-simplices with the incidence matrix $\bbB_{k}$, $k=1,\dots,K$. The rows of $\bbB_k$ are indexed by $(k-1)$-simplices and the columns by $k$-simplices. E.g., matrix $\bbB_1$ is the node-to-edge incidence matrix, and $\bbB_2$ is the edge-to-triangle incidence matrix \cite{schaub2021, yang2021finite}.

We can also use the Hodge Laplacian matrices, $\bbL_k = \bbB_k^\top\bbB_k + \bbB_{k+1}\bbB_{k+1}^\top$, where  $\bbL_{k,\rm{l}}\triangleq \bbB_k^\top \bbB_k$ and $\bbL_{k,\rm{u}}\triangleq \bbB_{k+1} \bbB_{k+1}^\top$ are the lower and the upper Laplacians, which encode lower and upper neighbourhoods, respectively. When $k=0$, the Hodge Laplacian is the graph Laplacian $\bbL_0 = \bbB_1 \bbB_1^\top$. 
Hodge Laplacians admit a \emph{Hodge decomposition}, leading to three orthogonal subspaces which the simplicial signal space can be decomposed into, i.e., 
$ \setR^{N_k} = \im(\bbB_k^\top) \oplus \im(\bbB_{k+1}) \oplus \ker(\bbL_k)$, where $\oplus$ is the direct sum of vector spaces and $\im(\cdot)$ and $\ker(\cdot)$ are the \emph{image} and \emph{kernel} of a matrix. For $k=1$, these subspaces carry the following interpretations \cite{barbarossa2020,schaub2021}.

\smallskip{\it Gradient space}. 
By applying matrix $\bbB_1$ to an edge flow $\bbx^1$, we compute its net flow at each node, $\bbB_1 \bbx^1$. The incidence matrix $\bbB_1$ is called a \emph{divergence operator}. Its adjoint $\bbB_1^\top$ differentiates a node signal $\bbx^0$ along the edges to induce an edge flow $\bbB_1^\top\bbx^0$, i.e., it is the \emph{gradient operator}. As a result, any flow within $\im(\bbB_1^\top)$ can be written as the gradient of a node signal $\bbx^0$, i.e., $\bbx^1 = \bbB_1^\top\bbx^0$. We call $\bbx^1 \in \im(\bbB_1^\top)$ a \emph{gradient flow} and the space $\im(\bbB_1^\top)$ the \textit{gradient space}. 

\smallskip{\it Curl space}. 
We can induce an edge flow from a triangle signal $\bbx^2$ as $\bbx^1 = \bbB_2\bbx^2$. The induced flow $\bbx^1 \in \im(\bbB_2)$ is called a \emph{curl flow} and the space $\im(\bbB_2)$ is the \textit{curl space}. The adjoint $\bbB_2^\top$ is the {\it curl operator}. We can use it to compute the net edge flow of $\bbx^1$ circulating along the triangles as $\bbB_2^\top\bbx^1$. 

\smallskip{\it Harmonic space}. 
The remaining space $\ker(\bbL_1)$ is the {\it harmonic space}. Any edge flow $\bbx^1\in\ker(\bbL_1)$ has zero divergence and curl, i.e., it is \emph{divergence-} and \emph{curl-free}. 

\smallskip Due to the boundary condition $\bbB_1 \bbB_2 = \mathbf{0}$, any gradient flow $\bbx^1\in\im(\bbB_1^\top)$ is curl-free. 
The space orthogonal to the gradient space, i.e., $\ker(\bbB_1) = \im(\bbB_2) \oplus \ker(\bbL_1)$, is called the \textit{cycle space}, which consists of both the curl space and harmonic space. Any flow in this space is divergence-free.

\vspace{1mm}\noindent\textbf{Simplicial filters.} To process simplicial signals $\bbx^1$, we use a simplicial convolutional filter of the following form:
\begin{equation} \label{eq.simplicial-filter}
    \bbH  = \epsilon\bbI + \sum_{l_1=1}^{L_1} \alpha_{l_1} (\bbB_1^\top\bbB_1)^{l_1} + \sum_{l_2=1}^{L_2} \beta_{l_2} (\bbB_2 \bbB_2^\top)^{l_2}
\end{equation}
where $\epsilon$, $\balpha = [\alpha_1,\dots,\alpha_{L_1}]^\top$ and $\bbeta = [\beta_1,\dots,\beta_{L_2}]^\top$ are the filter coefficients \cite{yang2021finite}. For ease of exposition, we only discuss the simplicial filter form \eqref{eq.simplicial-filter} for the edge signal space $\setR^{N_1}$, but similar discussions apply to general simplicial filters with $\bbB_k$ and $\bbB_{k+1}$. 

Applying $\bbH$ to an input edge flow $\bbx^1$ consists of the \emph{simplicial shifting} operations, $\bbL_{1,\rm{l}}\bbx^1$ and $\bbL_{1,\rm{u}}\bbx^1$. The $i$th entry of $\bbL_{1,\rm{l}}\bbx^1$ is
$
    [\bbL_{1,\rm{l}}\bbx^1]_i = \sum_{j\in\{\ccalN_{{\rm l},i} \cup \;i\}} [\bbL_{1,{\rm l}}]_{ij}[\bbx^1]_j
$,
which is a local operation within the lower neighborhood of the $i$th edge, likewise for $\bbL_{1,\rm{u}}\bbx^1$. Moreover, powers $\bbL_{1,\rm{l}}^k \bbx^1 = \bbL_{1,\rm{l}} (\bbL_{1,\rm{l}}^{k-1} \bbx^1)$ can be recursively obtained by applying the local operation $k$ times \cite{yang2021finite}. This leads to a distributed implementation of simplicial filtering with a complexity of order $\ccalO(N_1 D)$ for each shifting with $D$ being the maximal number of neighbours.

As we can observe from \eqref{eq.simplicial-filter}, in the simplicial domain, different sets of coefficients on $\bbL_{1,\rm{l}}$ and $\bbL_{1,\rm{u}}$ enable an independent and flexible filtering within the lower and upper simplicial neighbourhoods. 
When $L_1 = L_2$ and $\balpha = \bbeta$, filter $\bbH$ [cf.~\eqref{eq.simplicial-filter}] reduces to the basic form $\bbH = \sum_{l=0}^{L} h_{l} \bbL_1^l$ at the cost of losing expressive power and flexibility \cite{yang2021finite}. 

\section{Simplicial convolutional neural networks}
Upon having a simplicial convolutional filter \eqref{eq.simplicial-filter}, we can build an SCNN by composing filter banks with elementwise nonlinearities. This SCNN applies to any $k$-simplicial signal but we again illustrate it for edge signals $\bbx^1$ for the ease of intuition, and omit the superscript to avoid overcrowded notation.

Consider a $P$-layer SCNN. In the first layer $p=1$, we apply $F$ filters $\bbH_1^f$ [cf. \eqref{eq.simplicial-filter}] and an elementwise nonlinearity $\sigma(\cdot)$ to the input $\bbx_0$ to get a collection of $F$ features $\bbx_1^f$ as
\begin{equation} \label{eq.scnn-first-layer}
    \bbx_1^f = \sigma[\bbz_1^f] = \sigma[\bbH_1^f\bbx_0], \quad f = 1,\dots,F
\end{equation}
which constitute the output feature matrix $\bbX_1 = [\bbx_1^1,\dots,\bbx_1^{F}]$. In subsequent intermediate layers $p=2,\dots,P-1$, we have $\bbX_{p-1} = [\bbx_{p-1}^{1},\dots,\bbx_{p-1}^{F}] \in \setR^{N_1\times F}$ as input. Each input signal $\bbx_{p-1}^g, g = 1,\dots,F$ is passed through a bank of filters $\bbH_{p}^{fg}$ to obtain $F$ intermediate outputs $\bbz_p^{fg} = \bbH_{p}^{fg}\bbx_{p-1}^g, f = 1,\dots,F$. To avoid exponential filter growth, the intermediate outputs of the different input signals $\bbx_{p-1}^g$ are summed, thus, the $p$th layer generates $F$ features $\bbx_p^f$ as follows
\begin{equation} \label{eq.scnn-intermediate}
    \bbx_p^f = \sigma \bigg[ \sum_{g=1}^F \bbz_p^{fg} \bigg] = \sigma \bigg[ \sum_{g=1}^F \bbH_{p}^{fg}\bbx_{p-1}^g  \bigg]~ f = 1,\dots,F.
\end{equation}
The processing in \eqref{eq.scnn-intermediate} is repeated until the last layer $p=P$, where we consider the output has a single feature, and hence each input is processed by a single filter $\bbH^g$. Thus, the final output of the SCNN is given by
\begin{equation} \label{eq.scnn-last-layer}
    \bbx_P = \sigma\bigg[ \sum_{g=1}^F \bbz_P^g \bigg] = \sigma \bigg[ \sum_{g=1}^F \bbH_{P}^{g}\bbx_{P-1}^g  \bigg].
\end{equation}

Equations \eqref{eq.scnn-first-layer}, \eqref{eq.scnn-intermediate} and \eqref{eq.scnn-last-layer} constitute the SCNN architecture based on the simplicial filter form defined in \eqref{eq.simplicial-filter}. The lower and upper Hodge Laplacians encode the lower and upper simplicial neighbourhoods, respectively. Simplicial convolutions through filter \eqref{eq.simplicial-filter} are performed independently within the lower and upper neighbourhoods and controlled by different sets of coefficients. As we shall show later on, this means that the gradient and curl components of the input features are convolved independently, leading to more expressive power. Next, we discuss the connections of the SCNN with current alternatives and analyze its properties.

\vspace{1mm}\noindent\textbf{Links with related works.} In \cite{ebli2020simplicial}, a similar convolutional layer was proposed but based on filter $\bbH = \sum_{l=0}^Lh_l\bbL_1^l$. This SNN architecture is a particular case of the proposed SCNN with less expressive power but also with less parameters. 
The message passing neural network (MPNN) for simplicial complexes \cite{bodnar2021weisfeiler} aggregates and updates features from direct simplicial neighbours and simplices of different orders, e.g., nodes and triangles.
%
By considering an order-one simplicial convolution as the message aggregation step, we then obtain the architectures in \cite[Eq. 4]{bunch2020simplicial} and \cite[Eq. 7]{roddenberry2021principled}. When only edge features are available, such approaches are a particular case of the SCNN with order $L_1 = 1$ and $L_2=1$. Recurrent architectures are considered for flow interpolation and graph classification in \cite{roddenberry2019hodgenet}. Compared to these works, the SCNN treats features from the lower and upper neighbours differently and considers features from not only direct neighbours but also for multihop neighbors.

\vspace{1mm}\noindent\textbf{Locality and complexity.}
The intermediate output $\bbz^{fg}_p$ at the $p$th layer collects for each edge information from lower neighbours up to $L_1$ hops away and upper neighbours $L_2$ hops away through filter \eqref{eq.simplicial-filter}.
This locality comes from the structure of the Hodge Laplacian, likewise that of GNNs \cite{gama2020}. 

When only a single feature is available, we have $1+L_1+L_2$ parameters in such layers. For layers with multiple input and output features, the number of parameters grows $F^2$ times. The major complexity comes from the convolutional filtering step, which as seen before it is a weighted linear combination of different shifts of a simplicial signal; a local operation within the simplicial neighbourhoods that can be computed recursively. Hence, an SCNN layer performs the simplicial filtering for each edge with a cost of order $\ccalO((L_1+L_2)D)$. Again, this complexity grows $F^2$ times when multiple features are used and $P$ times if $P$ layers are considered.

\vspace{1mm}\noindent\textbf{Equivariance and invariance.}
In the following, we show that our SCNN is equivalent with respect to a different simplex labeling and different reference flow orientations. Consider the set of simplicial permutation matrices
$$
    \ccalP = \{\bbP_k \in \{0,1\}^{N_k\times N_k}: \bbP_k\bb1 = \bb1, \bbP_k^\top\bb1 = \bb1, k\geq 0 \},
$$
where products $\bbP_k\bbx^k$ permute the $k$-simplicial signal $\bbx^k$. Let $\bbP = (\bbP_0, \bbP_1,\dots)$ denote a sequence of permutations. Then, the following holds.
\begin{prop}
The SCNN is a permutation equivariant architecture. For an input edge flow $\bbx$, the output of an edge space SCNN layer with a simplicial filter $\bbH$, $\bby = \sigma[\bbH\bbx]$, becomes $\bby' = \bbP_1 \bby$ after a permutation sequence $\bbP$. 
\end{prop}
\begin{proof}
An SCNN layer with filter $\bbH$ gives the output $\bbz = \bbH\bbx$. After a sequence of permutations $\bbP$, the input edge flow $\bbx$ becomes $\bbP_1 \bbx$ and the boundary operators become $\bbP_0\bbB_1\bbP_1^\top$ and $\bbP_1\bbB_2\bbP_2^\top$. Thus, the Hodge Laplacians become $\bbP_1\bbL_{1,\rm{l}}\bbP_1^\top$ and $\bbP_1\bbL_{1,\rm{u}}\bbP_1^\top$ due to $\bbP_k^\top\bbP_k= \bbI$. Then we can express the permuted intermediate output as\vskip-5mm
\begin{equation*} \label{eq.permutation-equivariant}
\begin{aligned}
    \bbz' &\! =\! \bigg(\! \epsilon\bbI \!+\! \sum_{l_1=1}^{L_1}\! \alpha_{l_1} (\bbP_1\bbL_{1,\rm{l}}\bbP_1^\top)^{l_1} \!+\!\!\! \sum_{l_2=1}^{L_2}\! \beta_{l_2} (\bbP_1\bbL_{1,\rm{u}}\bbP_1^\top)^{l_2}\! \bigg) \bbP_1\bbx \\
    & = \bbP_1 \bigg( \epsilon\bbI + \sum_{l_1=1}^{L_1} \alpha_{l_1} \bbL_{1,\rm{l}}^{l_1} + \sum_{l_2=1}^{L_2} \beta_{l_2} \bbL_{1,\rm{u}}^{l_2} \bigg)\bbx = \bbP_1 \bbz.
\end{aligned}
\end{equation*}
Thus, the simplicial filter $\bbH$ is permutation equivariant. Furthermore, since the nonlinearity $\sigma(\cdot)$ is elementwise, SCNNs are permutation equivariant.
\end{proof}\vskip-2mm

In addition, in a simplicial complex, we have also set an arbitrary reference orientation for a simplex. A new reference orientation can be modelled by multiplying the rows and columns of the boundary matrices $\bbB_{k}$ and $\bbB_{k+1}$ where that $k$-simplex appears and the corresponding simplicial signal value by $-1$. Let then $\bbD_k$ be diagonal matrices from the set\vskip-5mm
$$
\mathcal{D} = \{\bbD_k = \diag(\bbd_k): \bbd_k \in \{\pm 1\}^{N_k}, k \ge 1, \bbd_0 = \mathbf{1}	\},
$$
where $\bbD_k \bbx^k$ is the updated $k$-simplicial signal $\bbx^k$. Let $\bbD = (\bbD_0,\bbD_1,\dots)$ denote a sequence of orientation changes. Then, the following holds.


\begin{prop}
The SCNN is orientation equivariant if the nonlinearity $\sigma(\cdot)$ is odd. Without loss of generality, for an input flow $\bbx$, the output of an edge space SCNN layer with a simplicial filter $\bbH$, $\bby = \sigma[\bbH \bbx]$ becomes $\bby'= \bbD_1\bby$ after a sequence of orientation changes $\bbD$.
\end{prop}

\begin{proof}
After an orientation change, the edge flow $\bbx$ becomes $\bbD_1\bbx$ and the boundary operators $\bbB_1$ and $\bbB_2$ are updated as $\bbD_0 \bbB_1 \bbD_1$ and $\bbD_1 \bbB_2 \bbD_2$. Then, the Hodge Laplacians become $\bbD_1 \bbL_{1,\rm{l}} \bbD_1$ and $\bbD_1 \bbL_{1,\rm{u}} \bbD_1$. We can then express the new filter output as $\bbz' = \bbD_1 \bbH \bbx = \bbD_1 \bbz$, following similar steps as in \eqref{eq.permutation-equivariant}. Thus, simplicial filter $\bbH$ is orientation equivariant. When the nonlinearity $\sigma(\cdot)$ is an odd function, we have $\bby' = \sigma(\bbz') = \bbD_1 \sigma(\bbz) = \bbD_1 \bby$ that completes the proof.
\end{proof} 

Permutation and orientation equivariances preserve the output of an SCNN regardless of the choices of the labeling and reference orientation of the edges. In turn, they allow the SCNN to exploit the internal symmetries in the complex. 

\vspace{1mm}\noindent\textbf{Spectral analysis.} For the spectral analysis, consider first that the eigenvectors of the Hodge Laplacian $\bbL_1$ span the three spaces given by the Hodge decomposition: 
(i) the gradient space $\im(\bbB_1^\top)$ is spanned by a set of eigenvectors $\bbU_{\rm{G}}$ of $\bbL_{1,\rm{l}}$ with positive eigenvalues; (ii) the curl space $\im(\bbB_2)$ is spanned by a set of eigenvectors $\bbU_{\rm{C}}$ of $\bbL_{1,\rm{u}}$ with positive eigenvalues; and (iii) the harmonic space $\ker(\bbL_1)$ is spanned by the eigenvectors $\bbU_{\rm{H}}$ of $\bbL_1$ with zero eigenvalue. Moreover, we have $\im(\bbL_1) = \im(\bbU_{\rm{G}}) \oplus \im(\bbU_{\rm{C}})$, i.e., gradient and curl spaces make up the image of $\bbL_1$ \cite{yang2021finite}. Then, we can eigendecompose $\bbL_1$ as $\bbL_1 = \bbU \bLambda \bbU^\top$
where $\bbU = [\bbU_{\rm{H}}\,\, \bbU_{\rm{G}} \,\, \bbU_{\rm{C}}]$ 
provides a simplicial Fourier basis, and $\bLambda = \diag(\bLambda_{\rm{H}}, \bLambda_{\rm{G}}, \bLambda_{\rm{C}})$ with $\bLambda_{\rm{H}} = \diag(\mathbf{0}_{N_H})$, $\bLambda_{\rm{G}} = \diag( \lambda_{\rm{G},1},\dots,\lambda_{{\rm{G}},N_{\rm{G}}})$, and  $\bLambda_{\rm{C}} = \diag(\lambda_{\rm{C},1},\dots,\lambda_{{\rm{C}},N_{\rm{C}}})$ collecting the harmonic, gradient, and curl frequencies, respectively; i.e., the simplicial frequencies \cite{yang2021finite}. 

For an edge flow $\bbx$, we can find its simplicial Fourier transform (SFT) as $\tilde{\bbx} = \bbU^\top \bbx$. This further defines three embeddings $\tilde{\bbx} = 
[\tilde{\bbx}_{\rm{H}}^\top \,\, \tilde{\bbx}_{\rm{G}}^\top \,\, \tilde{\bbx}_{\rm{C}}^\top]$: the harmonic embedding $\tilde{\bbx}_{\rm{H}} = \bbU^\top_{\rm{H}} \bbx$, the gradient embedding $\tilde{\bbx}_{\rm{G}} = \bbU^\top_{\rm{G}} \bbx$, and the curl embedding $\tilde{\bbx}_{\rm{C}} = \bbU^\top_{\rm{C}} \bbx$, which contain the weights of $\bbx$ at harmonic, gradient, and curl frequencies, respectively. 

\begin{figure}[t] 
    \vspace{-8mm}
  \centering
  \subfloat[$10\%$ missing rate. \label{fig:training_loss_10}]{%
       \includegraphics[width=0.48\linewidth]{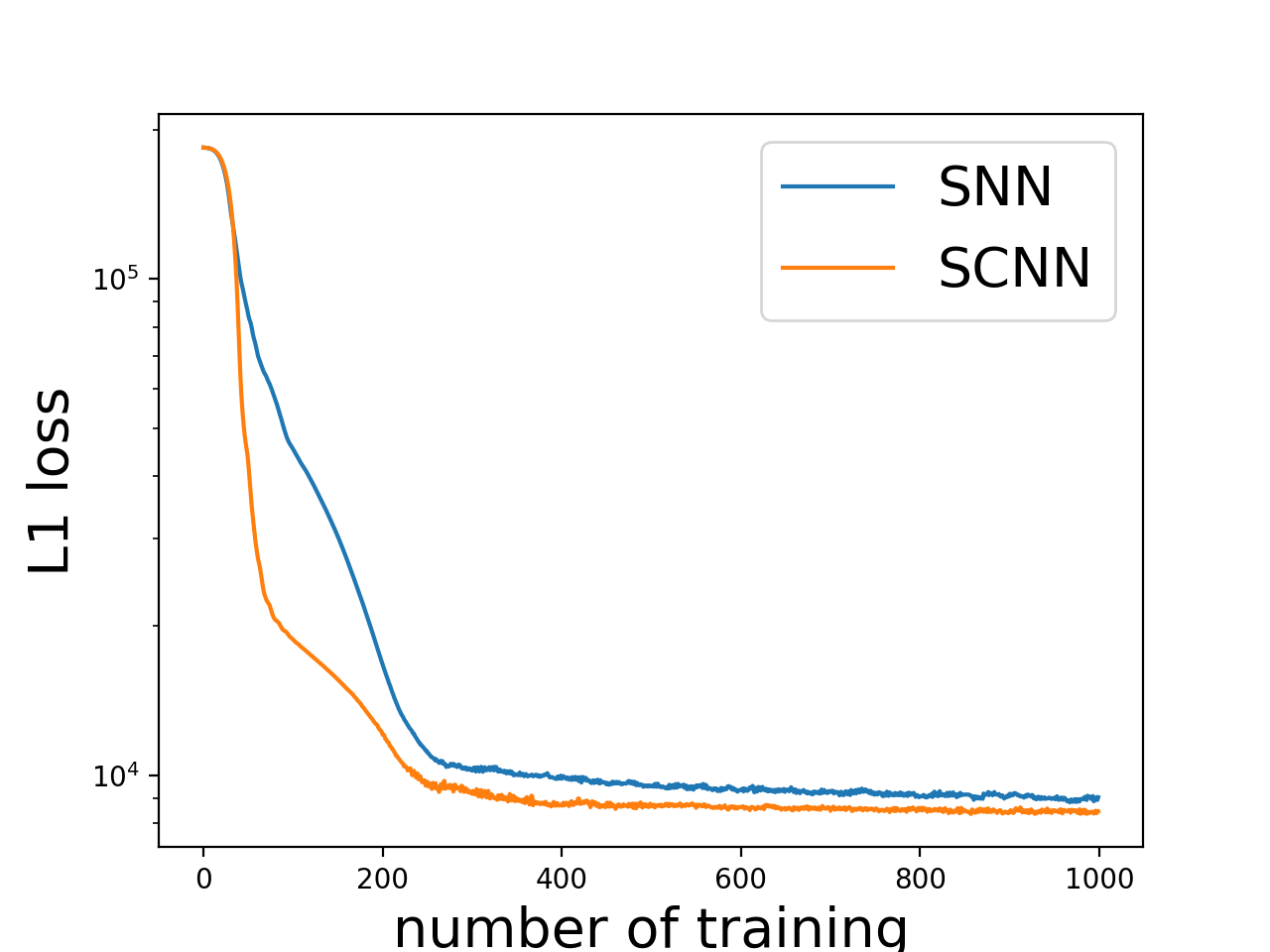}}
  \subfloat[$20\%$ missing rate.\label{fig:training_loss_20}]{%
        \includegraphics[width=0.48\linewidth]{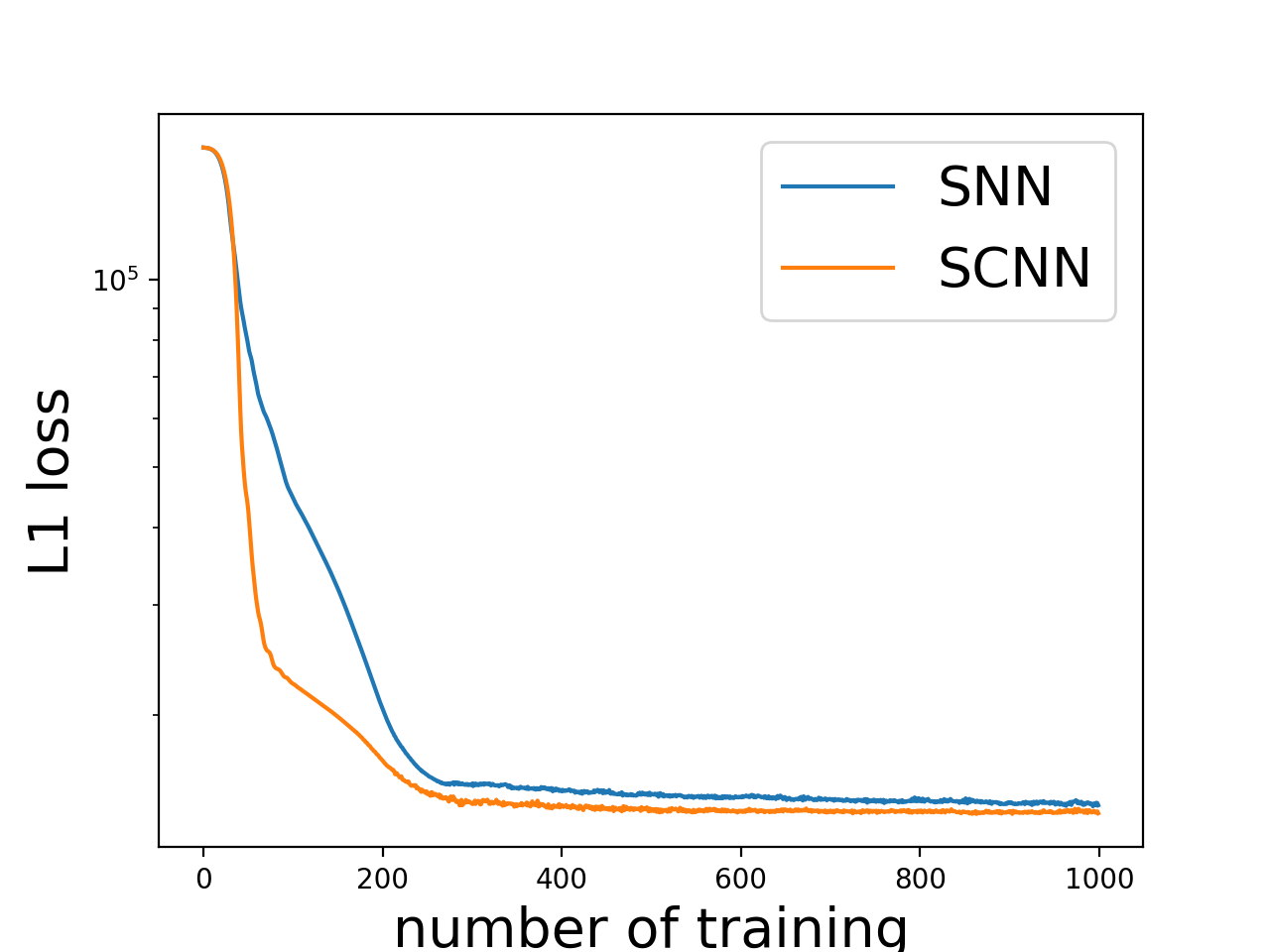}}
    \vspace{-3mm}
    \caption{Training loss of SNN and SCNN. We see that our SCNN converges much faster in the early training stage and results in a smaller training loss.}
    \label{fig:2}
    \vspace{-5mm}
\end{figure}

Via the eigendecomposition of $\bbL_1$, we can analyze the proposed SCNNs in the spectral domain. First, the frequency response of a simplicial convolutional filter $\bbH$ is given by
\begin{equation} \label{eq.freq-response}
    \tilde{H}(\lambda_i) = \begin{cases}
    \epsilon, &\text{ for } \lambda_i = 0, \\
     \epsilon + \sum_{l_1 = 1}^{L_1}\alpha_{l_1}\lambda_{i}^{l_1}, &\text{ for }  \lambda_i \in \ccalQ_{\rm{G}}, \\
    \epsilon + \sum_{l_2=1}^{L_2}\beta_{l_2}\lambda_{i}^{l_2}, &\text{ for }  \lambda_i \in \ccalQ_{\rm{C}}, \\
    \end{cases}
\end{equation}
where $\ccalQ_{\rm{G}}$ and $\ccalQ_{\rm{C}}$ respectively collect distinct gradient and curl frequencies. Thus, the $i$th entry of the SFT of the intermediate output $\bbz$ of an SCNN layer can be expressed as $\tilde{z}_i = \tilde{H}(\lambda_i)\tilde{x}_i$.
This spectral analysis shows that the SCNN layers compute the high-level simplicial components as the pointwise multiplication of the input embedding and the simplicial filter frequency response, ultimately respecting the convolution theorem. Furthermore, we have here different frequency responses for the gradient and curl components, which corresponds to the different weights on lower and upper Laplacians in filter $\bbH$. However, for a filter with $\boldsymbol{\alpha} = \boldsymbol{\beta}$ as in \cite{ebli2020simplicial}, this would lead to coupling between independent frequencies, and in turn to a limited learning expressiveness.



\section{Numerical results} \vskip-3mm
We use the SCNN to impute missing citations in a coauthorship complex, in which a paper with $k+1$ authors is represented by a $k$-simplex, and the $k$-simplicial signal is the number of citations of the paper. We followed the steps of \cite{ebli2020simplicial}, which lead to the citation dataset in Table \ref{tab:1}. We compared the SCNN with the SNN  in \cite{ebli2020simplicial} for the $k$-simplicial signals with $k=0,\dots,5$. Missing data are generated randomly on the $k$-simplicial signals at 5 rates, $10\%, 20\%,\dots,50\%$. The input of the SCNNs is the $k$-simplicial signal where missing citations are replaced by the median of known citations. As the SNN in \cite{ebli2020simplicial}, our SCNN has 3 layers with 30 convolutional filters of total length $5$ ($L_1=L_2=2$). We used LeakyReLU for $\sigma(\cdot)$ as in~\cite{ebli2020simplicial} although not odd. The reference orientation didn't seem to have much influence. We used the $\ell_1$ norm to train the NNs over known citations for 1000 iterations using the Adam optimizer with a learning rate $10^{-3}$. 

We report the training loss of two instances in Fig.~\ref{fig:2} and the mean accuracy\footnote{A citation value is considered to be correct if the imputed value is within $\pm5\%$ of the true value.} $\pm$ the standard deviation in Table~\ref{tab:1} over 10 different experiments. The proposed SCNN approach achieves a smaller training loss and a faster convergence than the SNN (Fig. \ref{fig:2}) due to its better expressive power. From Table~\ref{tab:1}, we observe that both NNs perform similarly for dimensions 0 and 1. This is because for the former, the two NNs are the same, and for the latter, the data dimension is rather small. However, the SCNN gives consistently $1-2\%$ better accuracies for $k\geq2$ with larger data dimensions. \vskip-4mm

\begin{table}[!t]
\vspace{-5mm}
\caption{Imputation accuracies for each dimension and missing rate by SNN (first rows) and SCNN (second rows).}
\centering
\resizebox{1.04\columnwidth}{!}{
{\scriptsize
\setlength{\tabcolsep}{0.35em}
\begin{tabular}{c|ccccccc}
\thickhline
\rowcolor{Gray} Order & 0 & 1 & 2 & 3 & 4 & 5  \\  
\rowcolor{Gray} $N_k$ & 352 & 1474 & 3285 & 5019 & 5559 & 4547  \\
\thickhline
$10\%$ & $0.91 \pm 0.003$  & $0.91 \pm 0.002$ & $0.90 \pm 0.004$ & $0.91 \pm 0.004$ & $0.90 \pm 0.016$ & $0.90 \pm 0.008$ \\ 
 $10\%$ & $0.91 \pm 0.004 $ & $0.91 \pm 0.002 $ & $0.91 \pm 0.002$ & $0.92 \pm 0.001$ & $0.92 \pm 0.002$ & $0.92 \pm 0.002$  \\ \hline
\rowcolor{Gray} $20\%$ & $0.81 \pm 0.006$ & $0.82 \pm 0.003$ & $0.82 \pm 0.005$ & $0.83 \pm 0.004$ & $0.82 \pm 0.012$ & $0.83 \pm 0.007$ \\ 
\rowcolor{Gray} $20\%$ & $0.81 \pm 0.007$ & $0.82 \pm 0.003$ & $0.83 \pm 0.003$ & $0.83 \pm 0.002$ & $0.84 \pm 0.002$ & $0.84 \pm 0.002$\\ \hline
$30\%$  & $0.72 \pm 0.006$ & $0.73 \pm 0.004$ & $0.73 \pm 0.005$ & $0.75 \pm 0.002$ & $0.75 \pm 0.002$ & $0.75 \pm 0.003$ \\ 
$30\%$  & $0.72 \pm 0.005$ & $0.73 \pm 0.004$ & $0.74 \pm 0.003$ & $0.75 \pm 0.002$ & $0.76 \pm 0.002$ & $0.77 \pm  0.002$\\ \hline
\rowcolor{Gray} $40\%$  & $0.63 \pm 0.007$ & $0.64 \pm 0.003$ & $0.65 \pm 0.003$ & $0.66 \pm 0.004$ & $0.67 \pm 0.009$ & $0.67 \pm 0.008$\\ 
\rowcolor{Gray} $40\%$  & $0.63 \pm 0.006$ & $0.64 \pm 0.003$ & $0.65 \pm 0.002$ &  $0.66 \pm 0.002$ &  $0.67 \pm 0.003$ & $0.69 \pm 0.002$\\ \hline
$50\%$  & $0.54 \pm 0.007$ & $0.55 \pm 0.005$ & $0.56\pm0.003$ & $0.57 \pm 0.003$ & $0.59 \pm 0.004$ & $0.60 \pm 0.005$  \\ 
$50\%$  & $0.54 \pm 0.006$ & $0.55 \pm 0.004$ & $0.56 \pm 0.003$ & $0.58 \pm 0.003$ & $0.59 \pm 0.003$ &  $0.61 \pm 0.002$ \\ \hline
\thickhline
\end{tabular}
}
}
\label{tab:1}
\vspace{-5mm}
\end{table}

\section{Conclusion} \vskip-3mm
This paper proposed a simplicial convolutional neural network architecture to learn from data defined on higher-order structures of a network, i.e., simplices in a simplicial complex. We built an SCNN layer through a composition of a simplicial filter and an elementwise nonlinearity. Due to the use of an advanced simplicial filter, our SCNN is able to learn from simplicial neighbours over multiple hops and process simplicial subcomponents (e.g., gradient and curl) independently, compared with the current solutions. The proposed SCNN applies to any $k$-simplicial signal case. We showed the SCNN is equivariant to permutations in the topology and to orientations in the flows, which allows it to exploit symmetries in a simplicial complex. In the future, we plan to extend the SCNN to include also the data on adjacent simplices.

\bibliographystyle{IEEEtran}
\bibliography{refs}

\end{document}